\documentclass{article}

\usepackage{neurips_2025}

\usepackage[utf8]{inputenc} 
\usepackage[T1]{fontenc}    
\usepackage{hyperref}       
\usepackage{url}            
\usepackage{booktabs}       
\usepackage{amsfonts}       
\usepackage{nicefrac}       
\usepackage{microtype}      
\usepackage{xcolor}         
\usepackage{graphicx} 
\usepackage{subfigure}
\usepackage{bm}
\usepackage{amsmath} 
\usepackage{multirow} 
\usepackage{makecell}
\usepackage{colortbl}
\usepackage{wrapfig}
\usepackage{soul}
\usepackage{tabularx}
\usepackage{algorithmic}
\usepackage{amsmath}
\usepackage{amsthm}
\usepackage{amssymb}
\usepackage[linesnumbered,ruled,vlined]{algorithm2e}
\newcommand{\bmmc}[1]{\bm{\mathcal{#1}}}

\usepackage{float}
\usepackage{mathtools} 

\newtheorem{proposition}{Proposition}
\newtheorem{definition}{Definition}
\newtheorem{remark}{Remark}

\newcommand{\normF}[1]{\|#1\|_F}      
\newcommand{\normtwo}[1]{\|#1\|_2}    
\newcommand{\R}{\mathbb{R}}           
\newcommand{\argmin}{\operatornamewithlimits{argmin}} 

\title{EcoSpa: Efficient Transformer Training with Coupled Sparsity}

\author{%
  \textbf{Jinqi Xiao}$^{1}$,\quad
  \textbf{Cheng Luo}$^{2}$,\quad
  \textbf{Lingyi Huang}$^{1}$,\quad
  \textbf{Cheng Yang}$^{1}$,\quad
  \textbf{Yang Sui}$^{3}$,\quad
  \textbf{Huy Phan}$^{1}$,\\\\
  \textbf{Xiao Zang}$^{1}$,\quad
  \textbf{Yibiao Ying}$^{1}$,\quad
  \textbf{Zhexiang Tang}$^{1}$,\quad
  \textbf{Anima Anandkumar}$^{2}$,\quad
  \textbf{Bo Yuan}$^{1}$\\\\
  $^{1}$Rutgers University \quad
  $^{2}$California Institute of Technology \quad 
  $^{3}$Rice University
}
\begin{document}

\maketitle

\begin{abstract}
Transformers have become the backbone of modern AI, yet their high computational demands pose critical system challenges. While sparse training offers efficiency gains, existing methods fail to preserve critical structural relationships between weight matrices that interact multiplicatively in attention and feed-forward layers. This oversight leads to performance degradation at high sparsity levels. We introduce EcoSpa, an efficient structured sparse training method that jointly evaluates and sparsifies coupled weight matrix pairs, preserving their interaction patterns through aligned row/column removal. EcoSpa introduces a new granularity for calibrating structural component importance and performs coupled estimation and sparsification across both pre-training and fine-tuning scenarios. Evaluations demonstrate substantial improvements: EcoSpa enables efficient training of LLaMA-1B with 50\% memory reduction and 21\% faster training, achieves $2.2\times$ model compression on GPT-2-Medium with $2.4$ lower perplexity, and delivers $1.6\times$ inference speedup. The approach uses standard PyTorch operations, requiring no custom hardware or kernels, making efficient transformer training accessible on commodity hardware.

\end{abstract}

\section{Introduction}
Transformers~\cite{vaswani2017attention} have become the dominant architecture for modern machine learning, powering breakthroughs in natural language processing~\cite{devlin-etal-2019-bert,openai2023gpt4,touvron2023llama}, computer vision~\cite{dosovitskiy2020image,liu2021swin}, and beyond. However, their computational demands pose critical system challenges. Training large models like LLaMA-7B requires extensive GPU resources~\cite{touvron2023llama}, while deployment demands high-end accelerators, limiting accessibility to well-resourced institutions and increasing operational costs for practical applications.

\textbf{Sparse training} has emerged as a promising solution to these challenges. By dynamically identifying and removing unimportant parameters during training, sparse methods can reduce computational and memory costs while maintaining model quality. As illustrated in Fig.~\ref{fig:ecospa_overview}, sparse training applies to both \textit{pre-training} scenarios (training from random initialization) and \textit{fine-tuning} scenarios (adapting pre-trained models), offering efficiency improvements across the model lifecycle. Recent works have explored various structured sparsity approaches for transformers: head-level sparsity~\cite{chen2021chasing,chen2021earlybert,ma2023llm}, layer-wise removal~\cite{chen2024compressing,men2024shortgpt}, and fine-grained row/column pruning~\cite{ashkboos2024slicegpt,van2023llm}.

\begin{figure*}
    \centering
    \includegraphics[width=\linewidth]{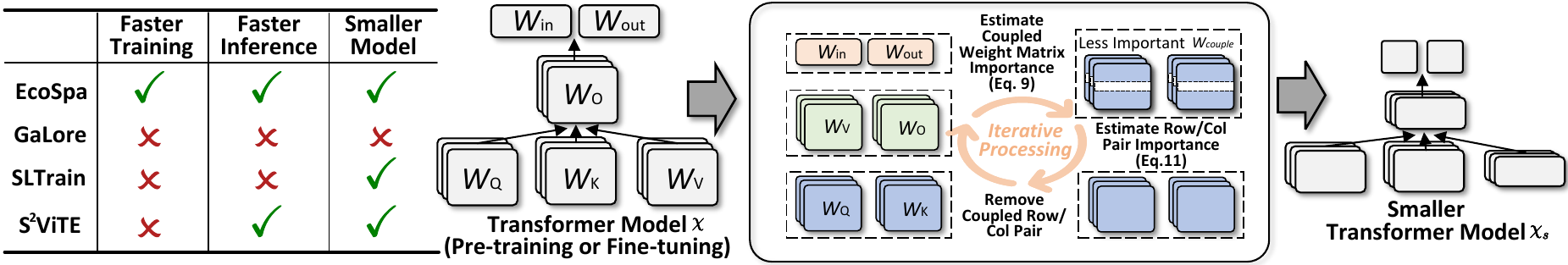}
    \caption{Overview of EcoSpa. Unlike existing methods that independently sparsify weight matrices, EcoSpa jointly estimates the importance of coupled weight matrices and removes aligned row/column pairs, preserving coupling relationships. The approach is applicable to both pre-training and fine-tuning scenarios.}
    \label{fig:ecospa_overview}
\end{figure*}

\textbf{The coupling problem.} Despite progress in sparse training for transformers~\cite{chen2021chasing,chen2021earlybert,ashkboos2024slicegpt,van2023llm}, most existing methods do not explicitly account for a critical architectural property: weight matrices in transformers often operate in \textit{coupled pairs} whose outputs interact multiplicatively. For example, in multi-head attention, $\bm{W}_Q$ and $\bm{W}_K$ produce queries and keys whose interaction determines attention scores. Similarly, feed-forward networks use consecutive matrices $\bm{W}_1$ and $\bm{W}_2$ whose outputs must be dimensionally compatible. Current approaches~\cite{ashkboos2024slicegpt,van2023llm} typically sparsify these coupled matrices by treating each matrix as an isolated component, which can disrupt their interaction and lead to performance degradation, especially at high sparsity levels where the remaining parameters must collectively preserve the model's representational capacity. This makes the preservation of their structural relationships critical for system efficiency.

\textbf{This paper} introduces EcoSpa, a structured sparse training method that preserves coupling relationships during sparsification. Rather than independently removing rows/columns from individual matrices, EcoSpa jointly evaluates coupled matrix pairs and removes \textit{aligned} row/column pairs---ensuring that when a column is removed from $\bm{W}_Q$, the corresponding column is also removed from $\bm{W}_K$. This alignment preserves the dimensionality and interaction of their outputs while reducing memory footprint and computational cost. We adapt empirical Fisher information to the context of coupled matrix pairs, enabling efficient identification of which pairs contribute least to model output, and then apply coordinated removal within each pair. Theoretical analysis (Appendix~\ref{theo_sec}) shows this approach better preserves the approximation quality of matrix products compared to independent sparsification.

\textbf{Contributions and Results.} We introduce two key innovations: (\textit{i}) \textit{Coupled estimation} that assesses importance at the granularity of coupled weight matrix pairs rather than individual matrices, and (\textit{ii}) \textit{Coupled sparsification} that removes aligned row/column pairs to preserve coupling relationships and dimensional compatibility. Our comprehensive evaluations demonstrate substantial system improvements: EcoSpa reduces GPU memory by 50\% (3.88GB vs 7.80GB) and training time by 21\% for LLaMA-1B pre-training, enabling efficient deployment while achieving $1.6\times$ inference speedup. For GPT-2-Medium, it delivers $2.2\times$ compression with $2.4$ lower perplexity and 40\% faster training. Vision Transformers achieve higher accuracy than prior sparse training methods at comparable compression levels. Ablation studies validate that coupled sparsification substantially outperforms independent approaches, with performance gaps widening at high sparsity levels where efficiency matters most for practical deployment.

\section{Background}
\subsection{Preliminaries}
\textbf{Notation.} We represent tensors using boldface calligraphic script, denoted as $\bmmc{X}$. Matrices and vectors are indicated with boldface capital and lowercase letters, such as $\bm{X}$ and $\bm{x}$, respectively. Furthermore, non-boldface letters with indices, \emph{e.g.}, $\mathcal{X}(i_1,\cdots,i_d)$, $X(i,j)$, and $x(i)$, denote the entries for a $d$-dimensional tensor $\bm{\mathcal{X}}$, a matrix $\bm{X}$, and a vector $\bm{x}$, respectively.

\textbf{Transformer.} For Transformer-based models, the key components include the Multi-Head Attention (MHA) and Feed-Forward Network (FFN). More specifically, the MHA operation is defined as follows:
\begin{equation}
{\rm MHA}(\bm{X}_Q,\bm{X}_K,\bm{X}_V)={\rm Concat}(head_1, \dots, head_h)\bm{W}^O,
\label{eq:mha}
\end{equation}
where $\bm{X}_Q,\bm{X}_K,\bm{X}_V \in \mathbb{R}^{l \times d_m}$ are the length-$l$ input sequences, $d_m$ is the embedding dimension, and $h$ is the number of attention heads. Here each attention head $head_i$ operates as below:
\begin{equation}
\begin{aligned}
head_i={\rm Attention}( \bm{X}_Q\bm{W}_{i}^{Q},\ \bm{X}_K\bm{W}_{i}^{K},\ \bm{X}_V\bm{W}_{i}^{V}) =\phi\left(\frac{\bm{X}_Q\bm{W}_{i}^{Q}(\bm{X}_K\bm{W}_{i}^{K})^\top}{\sqrt{d_h}}\right)\bm{X}_V\bm{W}_{i}^{V},
\label{eq:self-attention}
\end{aligned}
\end{equation}
where $\phi(\cdot)$ represents the softmax function, $\bm{W}_{i}^{Q}$, $\bm{W}_{i}^{K}, \bm{W}_{i}^{V} \in \mathbb{R}^{d_m \times d_h}$, $\bm{W}^{O} \in \mathbb{R}^{d_m \times d_m}$, and $d_m=d_h\times h$. 
The FFN consists of two fully connected layers with a Gaussian Error Linear Unit (GELU) activation function applied in between.
Let $\bm{X}$ represent the input embeddings, and $\bm{W}_{in}\in \mathbb{R}^{d_m \times d}, \bm{W}_{out}\in \mathbb{R}^{d \times d_m}, \bm{b}_{in}, \bm{b}_{out}$ denote the weight matrices and bias vectors, respectively. The operation of FFN can be then defined as follows:
\begin{equation}
{\rm FFN}(\bm{X}) = {\rm GELU}(\bm{X}\bm{W}_{in}+\bm{b}_{in})\bm{W}_{out}+\bm{b}_{out}.
\label{eq:ffn}
\end{equation}

\subsection{Related Work}

\textbf{Sparse Training for Transformers.} Sparse training for transformers aims to reduce computational and memory costs by identifying and removing unimportant parameters during training. Existing methods can be categorized by their target training stage and sparsification granularity.

\textit{Pre-training from Scratch.} Training sparse transformers from random initialization remains under-explored compared to sparse CNN pre-training, where dynamic sparse training methods~\cite{mocanu2018scalable, mostafa2019parameter, evci2020rigging, yuan2021mest, chen2023otov2} and structured sparsification~\cite{he2019filter, wang2019cop, lin2020hrank, hou2022chex} have been extensively studied. For transformers, Chen et al.~\cite{chen2021chasing} dynamically extract and train sparse sub-networks for Vision Transformers through unstructured or structured sparsification, alleviating memory bottlenecks. Chen et al.~\cite{chen2021earlybert} identify structured winning lottery tickets early in BERT pre-training to improve efficiency. Dao et al.~\cite{dao2022monarch} propose Monarch matrices---a class of structured matrices that approximate dense weights for hardware-efficient pre-training. Han et al.~\cite{han2024sltrain} combine low-rank and unstructured sparse matrices for the pre-training model. Low-rank compression with hardware awareness~\cite{xiao2023haloc} and attention-centric customization frameworks~\cite{xiao2023comcat} further demonstrate the benefits of coordinating structural assumptions across coupled transformer blocks. These methods reduce training costs but typically operate at coarse granularities (entire heads or layers) or apply sparsification independently to each weight matrix, with limited consideration of interactions between coupled matrices.

\textit{Fine-tuning Pre-trained Models.} Sparsifying pre-trained models for efficient deployment has attracted significant attention. Methods fall into two categories: \textit{unstructured sparsification}~\cite{frantar2023sparsegpt, sun2023simple, zhang2023dynamic, xia2023sheared, malla2024copal} removes individual weights, achieving high compression ratios but limited speedup due to irregular patterns requiring specialized hardware; \textit{structured sparsification}~\cite{ma2023llm, xia2023sheared, ashkboos2024slicegpt, an2024fluctuation, chen2023otov, chen2023lorashear, zhao2024apt, yang2024laco} targets architectural components (heads, rows, columns, layers), enabling practical speedup on commodity hardware. Complementary work explores memory-efficient optimization via correlation-aware gradient projections~\cite{xiao2025coap}, compatible token pruning for multimodal transformers~\cite{yang2025topv}, and expert-wise compression strategies for mixture-of-experts architectures~\cite{yang2024moe}. For importance estimation, magnitude-based approaches~\cite{sun2023simple, an2024fluctuation} use absolute weight values, while loss-based approaches~\cite{ma2023llm, van2023llm} assess sparsification impact via gradient information from Taylor expansion. However, most of these methods evaluate each weight matrix independently, which may disrupt multiplicative interactions between coupled matrix pairs in attention and feed-forward layers.

\textbf{Importance Estimation for Sparsification.} Determining which parameters to remove is central to sparse training. Magnitude-based methods~\cite{han2015deep, sun2023simple} use weight magnitudes as importance proxies, offering computational efficiency but often failing to capture actual contributions to model outputs. Gradient and loss-based methods~\cite{ma2023llm, van2023llm} estimate importance by analyzing how parameter removal affects loss through Taylor expansion, providing more accurate assessment while still treating parameters independently. Fisher information-based methods offer a principled approach by quantifying output sensitivity to parameter perturbations~\cite{theis2018faster}. However, existing applications compute importance for individual parameters or matrices separately, without jointly evaluating coupled matrix pairs whose outputs interact multiplicatively---a gap that becomes critical at high sparsity levels where preserving structural relationships is essential.

\section{Method}
EcoSpa consists of two key steps illustrated in Fig.~\ref{fig:fine-grained}. \textbf{(1) Coupled Estimation} (Section~\ref{subsec:estimation}) identifies unimportant coupled weight matrices by jointly evaluating matrix pairs that interact multiplicatively in transformer computations. \textbf{(2) Coupled Sparsification} (Section~\ref{subsec:Sparsification}) removes aligned row/column pairs from these coupled matrices to preserve their structural relationships while achieving sparsification.

\begin{figure*}[t]
    \centering
    \includegraphics[width=\linewidth]{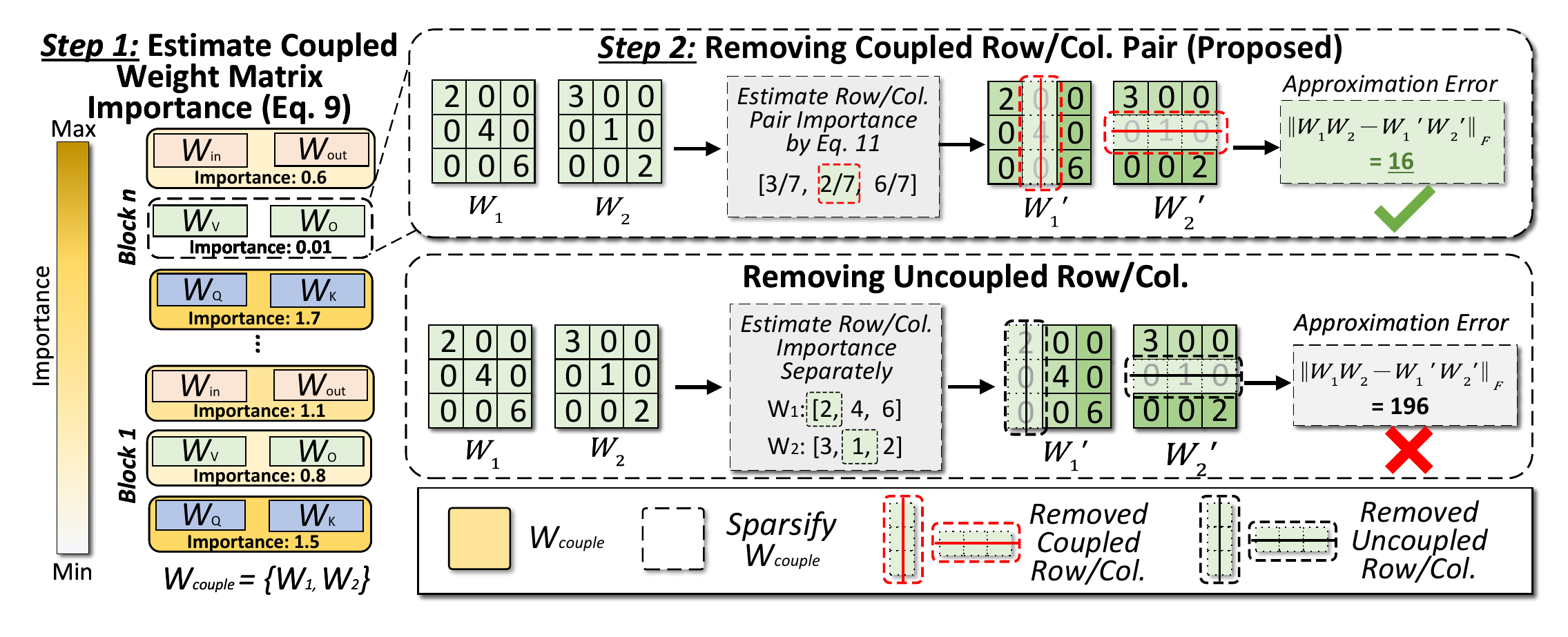}
    \caption{\textbf{EcoSpa Method Overview.} (Step 1) \textbf{Coupled Estimation}: Jointly evaluate importance of coupled weight matrix pairs using empirical Fisher information. (Step 2) \textbf{Coupled Sparsification}: Remove aligned row/column pairs from less important coupled matrices, preserving dimensional compatibility and interaction patterns. The method ensures that when a column is removed from $\bm{W}_1$, the corresponding row is also removed from $\bm{W}_2$.}
    \label{fig:fine-grained}
\end{figure*}

\subsection{Coupled Estimation: Coupled Weight Matrix-wise Importance Calibration}
\label{subsec:estimation}

As extensively studied in the literature~\cite{frantar2023sparsegpt, sun2023simple, van2023llm}, estimating the unimportant components of neural networks plays a crucial role for model sparsification. In general, importance estimation can be performed at different granularity levels, such as weight, neuron, layer, and head. More specifically, when aiming for obtaining the structured sparse transformers, identifying the insignificant heads and layers is the most common practice~\cite{chen2021chasing,chen2021earlybert,men2024shortgpt,ma2023llm}. 



Unlike existing works that independently evaluate individual weight matrices, we observe that matrices in transformers often operate in \textit{coupled pairs} whose outputs interact multiplicatively. For example, in multi-head attention (Eq.~\ref{eq:self-attention}), $\bm{W}_i^Q$ and $\bm{W}_i^K$ produce queries and keys that compute attention scores via their product $(\bm{X}_Q\bm{W}_i^Q)(\bm{X}_K\bm{W}_i^K)^\top$. Independently removing dimensions from $\bm{W}_i^Q$ (e.g., removing the $j$-th column) without corresponding removal from $\bm{W}_i^K$ disrupts dimensional compatibility and degrades the quality of their interaction. This observation motivates us to propose a new granularity for importance estimation: the \textit{coupled weight matrix}.

\textbf{Coupled Weight Matrix in MHA.} Recall that Eq. \ref{eq:mha} and Eq. \ref{eq:self-attention} depict the computations of MHA, which consists of four types of weight matrices  $\bm{W}^Q$, $\bm{W}^K$, $\bm{W}^V$ and $\bm{W}^O$. We propose to estimate the structural importance of MHA by analyzing the combined effect of these matrices. More specifically, consider the following mathematical reformulation of Eq. \ref{eq:self-attention}:
\begin{equation}
\begin{aligned}
{\rm MHA}(\bm{X}_Q,\bm{X}_K,\bm{X}_V)=\sum_{i=1}^{h} head_i\bm{W}_{i}^{O}=\sum_{i=1}^{h}\phi(\frac{\bm{X}_Q\overbrace{{\bm{W}_{i}}^{Q}{\bm{W}_{i}^{K}}^\top}^{\bm{W}_{i}^{QK}}\bm{X}^\top_K}{\sqrt{d_h}})\bm{X}_V \overbrace{{\bm{W}_{i}}^V {\bm{W}_{i}}^O}^{\bm{W}_{i}^{VO}},
\label{eq:self-attention-new}
\end{aligned}
\end{equation}
where $\bm{W}_{i}^{O} \in \mathbb{R}^{d_h \times d_m}$, $\bm{W}^{O}={\rm Concat}(\bm{W}_{1}^{O},\dots,\bm{W}_{h}^{O})$. It is seen that $\bm{W}_{i}^{QK}=\bm{W}_{i}^{Q}{\bm{W}_{i}^{K}}^{\top}$ and $\bm{W}_{i}^{VO}=\bm{W}_{i}^{V}\bm{W}_{i}^{O}$, as the combination of two weight matrices, can serve as the structural components in MHA. Following this perspective, we propose to set the granularity of importance estimation at the level of those coupled weight matrices. We believe this strategy brings two benefits: \textit{i}) it provides more fine-grained measurement than the commonly adopted head-level calibration; and \textit{ii}) it meanwhile naturally explores the inter-matrix correlation within the attention heads, avoiding the limitations if only focusing on individual weight matrices.

\textbf{Connection to Transformer Circuits Framework.} These coupled weight matrices correspond to distinct functional roles in the transformer circuits framework~\cite{elhage2021mathematical}: $\bm{W}_i^{QK}$ determines attention patterns by computing token-to-token relationships, while $\bm{W}_i^{VO}$ modulates attended information for output generation. Preserving these structural relationships during sparsification maintains the computational circuits essential for transformer effectiveness\footnote{For example, induction heads that enable in-context learning rely on $\bm{W}_i^{QK}$ for pattern retrieval and $\bm{W}_i^{VO}$ for correct prediction~\cite{olsson2022context}. Given an input like ``A B ... A'', $\bm{W}_i^{QK}$ computes attention scores to retrieve previous patterns, while $\bm{W}_i^{VO}$ ensures correct prediction of ``B'' as the next token.}.



\textbf{Extension to GQA.} Next we show that the concept of coupled weight matrix can also be applied to Grouped Query Attention (GQA) \cite{ainslie2023gqa} -- the generalization of MHA that has been adopted in  state-of-the-art LLMs such as LLaMA~\cite{touvron2023llama}, Falcon~\cite{almazrouei2023falcon}, Mistral~\cite{jiang2023mistral}. Recall that GQA divides $h$ attention heads into multiple groups, so its attention mechanism can be reformulated as follows:
\begin{equation}
\begin{aligned}
{\rm GQA}(\bm{X}_Q, \bm{X}_K, \bm{X}_V)
=\sum_{i=1}^{h}\phi(\frac{\bm{X}_Q\overbrace{\bm{W}_{i }^Q\bm{W}_{g(i)}^{K^\top}}^{\bm{W}_{i,g(i)}^{QK}}\bm{X}_K^\top}{\sqrt{d_h}})\bm{X}_V\overbrace{\bm{W}_{g(i)}^V \bm{W}_{i}^{O}}^{\bm{W}_{g(i),i}^{VO}},
\end{aligned}
\end{equation}
where $g(i)$ maps the $i$-th head to its corresponding group. It is seen the structural components in the format of coupled weight matrices also exist for GQA, as  $\bm{W}_{i,g(i)}^{QK}=\bm{W}_{i}^{Q}{\bm{W}_{g(i)}^{K^{\top}}}$ and $\bm{W}_{g(i),i}^{VO}=\bm{W}_{g(i)}^{V}\bm{W}_{i}^{O}$.

\textbf{Compatibility with RoPE.} Furthermore, when rotary position embedding (RoPE)~\cite{su2024roformer} is used in the model architecture, the  coupled weight matrix can also be incorporated into this position-aware scenario, capturing both the structural coupling and positional encoding as follows:
\begin{equation}
\begin{aligned}
&{\rm MHA}(\bm{X}_Q,\bm{X}_K,\bm{X}_V)_{\rm{RoPE}}
=\sum_{i=1}^{h}\phi(\frac{{\rm RoPE}(\bm{X}_Q{{\bm{W}_{i}}^{Q}){{\rm RoPE}({\bm{X}_K}\bm{W}_{i}^{K}}})^\top}{\sqrt{d_h}})\bm{X}_V\bm{W}_{i}^{VO}\\
&=\sum_{i=1}^{h}\phi(\frac{\bm{X}_Q\bm{W}_i^Q\bm{P}_t\bm{P}_s^\top\bm{W}_i^{K^\top}\bm{X}_K^\top}{\sqrt{d_h}})\bm{X}_V\bm{W}_{i}^{VO}
=\sum_{i=1}^{h}\phi(\frac{\bm{X}_Q\overbrace{\bm{W}_{i, {\rm RoPE}}^Q\bm{W}_{i, {\rm RoPE}}^{K^\top}}^{\bm{W}_{i, RoPE}^{QK}}\bm{X}_K^\top}{\sqrt{d_h}})\bm{X}_V\bm{W}_{i}^{VO},
\label{eq:self-attention-rope}
\end{aligned}
\end{equation}
where $\bm{W}_{i,\rm RoPE}^{Q}=\bm{W}_{i}^{Q}\bm{P}_{t}$ and $\bm{W}_{i,\rm 
 RoPE}^{K}=\bm{W}_{i}^{K}\bm{P}_{s}$. Here ${\rm RoPE}(\cdot)$ encodes positional information by applying position-dependent rotations to query and key projections, where $\bm{P}_t, \bm{P}_s \in \mathbb{R}^{d_h \times d_h}$ are rotation matrices corresponding to the $t$-th and $s$-th positions, respectively. Each matrix is block-diagonal, with blocks defined as:$\begin{pmatrix}
  \cos(t\theta_j) & \sin(t\theta_j)\\
 -\sin(t\theta_j) & \cos(t\theta_j)
\end{pmatrix}$, $\begin{pmatrix}
  \cos(s\theta_j) & \sin(s\theta_j)\\
 -\sin(s\theta_j) & \cos(s\theta_j)
\end{pmatrix}$, where $j \in \{1 \cdots d_h/2\}$ and ${\theta_j}$ is pre-defined frequency parameter, \emph{e.g.}, $\theta_j=1/10000^{2j/d_h}$. From Eq. \ref{eq:self-attention-rope} it is seen that with RoPE integrated, the positional encoding is naturally incorporated into the weight transformations, redefining the coupled weight matrix as $\bm{W}_{i,RoPE}^{QK}=\bm{W}_{i, {\rm RoPE}}^Q\bm{W}_{i, {\rm RoPE}}^{K^\top}$.


\textbf{Coupled Weight Matrix in FFN.} Following the same principle, we evaluate FFN importance through coupled matrices. For FFN with GELU activation (Eq.~\ref{eq:ffn}), removing the $j$-th column of $\bm{W}_{\rm in}$ and the $j$-th row of $\bm{W}_{\rm out}$ introduces an approximation error. We analyze this error using first-order Taylor expansion of GELU around zero ($\text{GELU}(u) \approx c u$, where $c$ is the derivative at zero):
\begin{equation}
\begin{aligned}
\text{Error}(j, \bm{X}) &\approx c \| \bm{X} (\bm{w}_{\rm in,j} \bm{w}_{\rm out,j}^T) \|_F \le c \|\bm{X}\|_F \|\bm{w}_{\rm in,j}\|_2 \|\bm{w}_{\rm out,j}\|_2
\end{aligned}
\label{eq:ffn_error}
\end{equation}
This bound shows the true removal error is proportional to the product $\|\bm{w}_{\rm in,j}\|_2 \|\bm{w}_{\rm out,j}\|_2$---precisely our coupled metric. The data-dependent factor $c \|\bm{X}\|_F$ is a common scaling term across all dimensions, making the coupled norm product the decisive factor for importance ranking. Therefore, we approximate FFN structural importance by evaluating the coupled matrix $\bm{W}_{\rm io}=\bm{W}_{\rm in}\bm{W}_{\rm out}$. While this is a first-order approximation, Table~\ref{tab:coupled_vs_uncoupled} empirically validates that coupled sparsification substantially outperforms independent approaches on FFN matrices, confirming the effectiveness of this strategy. Detailed derivation and extension to other activation functions (ReLU, SiLU, SwiGLU) are provided in Appendix~\ref{theo_sec_ffn}.

\textbf{Coupled Weight Matrix-wise Importance.} To estimate the structural criticality of coupled matrices, we adapt empirical Fisher information to jointly evaluate matrix pairs. Standard empirical Fisher information measures parameter importance by gradient magnitudes~\cite{kunstner2019limitations}. For a coupled matrix $\bm{W}_{\rm couple}=\bm{W}_1\bm{W}_2$, we estimate its importance by averaging the normalized gradient information of both component matrices:
\begin{equation}
    \hat{I}(\bm{W}_{\rm couple})=\frac{1}{2}\sum_{k=1}^{2} \frac{1}{|\bm{W}_{k}|} \frac{1}{|\bmmc{D}|} \sum_{i,j}(\frac{\partial \mathcal{L}(\bm{W}_{k};\bmmc{D})}{\partial \bm{W}_{k}})_{i,j}^2,
\label{eq:importance}
\end{equation}
where $|\bm{W}_k|$ denotes the number of elements in $\bm{W}_k$ and averaging ensures balanced contribution from both matrices regardless of their dimensions. Here for MHA, $\{\bm{W}_{1},\bm{W}_{2}\}$ is  $\{\bm{W}_{i}^Q,\bm{W}_{i}^{K^\top}\}$ and $\{\bm{W}_{i}^V,\bm{W}_{i}^O\}$; while for FFN, $\{\bm{W}_{1},\bm{W}_{2}\}$ is $\{\bm{W}_{\rm in},\bm{W}_{\rm out}\}$\footnote{FFN in LLaMA models contains three weight matrices ($\bm{W}_{\rm gate}$, $\bm{W}_{\rm up}$, and $\bm{W}_{\rm down}$) and SwiGLU. The coupled structure in this scenario is established by defining $\bm{W}_1=\bm{W}_{\rm gate}\odot\bm{W}_{\rm up}$, $\bm{W}_2=\bm{W}_{\rm down}$, preserving the key relationships between the weight matrices.}. Also, $\bmmc{D}$ is the training data. Different from prior works~\cite{hsu2022language,sung2021training} that compute Fisher information for individual parameters, we calculate it for the entire coupled weight matrix $\bm{W}_{\rm couple}$. Smaller $\hat{I}(\bm{W}_{\rm couple})$ indicates lower importance when sparsifying the model.

\subsection{Coupled Sparsification: Removing the Coupled Row/Column Pair}
\label{subsec:Sparsification}

Having identified less important coupled weight matrices (Section~\ref{subsec:estimation}), we now determine which row/column pairs to remove within them. Rather than removing entire coupled matrices (coarse-grained), we perform row/column-wise sparsification for finer control while maintaining coupling relationships. The key challenges are: (1) which pairs to prioritize for removal, and (2) how to ensure aligned removal across $\bm{W}_1$ and $\bm{W}_2$ to preserve their structural interaction.


\textbf{Inspiration from tSVD}. To determine which row/column pairs to remove from coupled matrices $\bm{W}_1$ and $\bm{W}_2$, we draw inspiration from truncated Singular Value Decomposition (tSVD). While directly applying SVD to $\bm{W}_1\bm{W}_2$ is mathematically infeasible (since ${\rm SVD}(\bm{W}_1\bm{W}_2)$ cannot uniquely recover $\bm{W}_1$ and $\bm{W}_2$), the \textit{principle} of tSVD provides valuable insights.

Consider a rank-$r$ approximation of matrix $\bm{M}$ via tSVD:
\begin{equation}
\bm{M} = \bm{U}\bm{\Sigma}\bm{V}^\top \approx (\bm{U}_{:,1:r}\bm{\Sigma}_{1:r,1:r}^{1/2})(\bm{\Sigma}_{1:r,1:r}^{1/2}\bm{V}_{:,1:r}^\top)
\label{eqn:tsvd}
\end{equation}
This optimal approximation~\cite{wall2003singular} exhibits two key properties: (1) removed singular components have small $\ell_2$ norms (proportional to smallest singular values), and (2) removal is \textit{aligned}---the $i$-th dimension is removed from both factor matrices simultaneously.

\textbf{Applying tSVD Principles.} Motivated by these observations, we design importance scores for row/column pairs in $\bm{W}_1$ and $\bm{W}_2$ that capture both properties:

\begin{equation}
\begin{aligned}
    \bm{\upsilon} =[\upsilon_1, \upsilon_2,\ldots ,\upsilon_d] = \frac{\left \| {\rm vec}(\bm{W}_1) \right \|_{\rm 2,col} \odot \left \| {\rm vec}(\bm{W}_2) \right \|_{{\rm 2,row}}}{\left \|\| {\rm vec}(\bm{W}_1) \right \|_{\rm 2,col} \odot \left \| {\rm vec}(\bm{W}_2) \right \|_{{\rm 2,row}}\|_2}
     \label{eq:l2norm}
\end{aligned}
\end{equation}
where $\left \| {\rm vec}(\cdot) \right \|_{\rm 2,col}$ and $\left \| {\rm vec}(\cdot) \right \|_{{\rm 2,row}}$ represent the column-wise and row-wise $\ell_2$ norm of a matrix, respectively. The Hadamard product $\odot$ combines the $\ell_2$ norms of corresponding dimensions (analogous to singular values in Property 1), and normalization enables direct comparison across different matrix pairs. This formulation ensures aligned removal (Property 2) while prioritizing pairs with small combined norms (Property 1). We can rank $\upsilon_i$'s to determine the important coupled row/column pairs in $\bm{W}_1$ and $\bm{W}_2$, and then remove those insignificant pairs.\footnote{When RoPE is used in attention, position encoding matrices $\bm{P}_t$ and $\bm{P}_s^\top$ are incorporated in importance calculation, while actual removal is performed on weight matrices $\bm{W}_i^Q$ and $\bm{W}_i^K$. Position matrices are then adjusted to match new dimensions.}

\textbf{Theoretical Justification.} We prove in Appendix~\ref{theo_sec} that selecting pairs with minimum $\|\bm{w}_1^{(j)}\|_2 \|\bm{w}_2^{(j)}\|_2$ minimizes the Frobenius norm approximation error $\|\bm{W}_1\bm{W}_2 - \tilde{\bm{W}}_1\tilde{\bm{W}}_2\|_F$ under aligned sparsification. Empirical validation on 5,000 random matrix pairs confirms coupled sparsification achieves lower error than uncoupled baselines with $>$99\% probability. We further validate this on pre-trained transformer models in Table~\ref{tab:coupled_vs_uncoupled}, where coupled sparsification substantially outperforms uncoupled approaches on FFN matrices across different sparsity levels.


\textbf{Overall Training Procedure.} Having established the coupled sparsification strategy---identifying unimportant coupled matrices (Eq.~\ref{eq:importance}) and removing aligned row/column pairs (Eq.~\ref{eq:l2norm})---we now describe how EcoSpa integrates this into the training process. Algorithm~\ref{alg:training} details the complete procedure. Key hyper-parameters include: (1) \textit{Target size} $c$: desired parameter count after sparsification; (2) \textit{top-$K$ ratio}: percentage of coupled matrices to sparsify in each iteration (focusing on least important ones accelerates convergence); (3) \textit{Cumulative threshold} $\theta$: controls how many row/column pairs to remove from each selected coupled matrix, where smaller values lead to more aggressive sparsification. In each epoch, after identifying top-$K$ least significant $\bm{W}_{\rm couple}$, the importance scores of the row/column pairs in those component weight matrices are calculated. Once the cumulative importance scores exceed the threshold $\theta$, the row/column pairs corresponding to the smallest score is removed. This gradual sparsification process continues till the overall model reaches the target budget size.

\begin{algorithm}[t]
\small
\caption{Processing Scheme of EcoSpa}
\label{alg:training}
\SetKwInOut{Input}{Input}
\SetKwInOut{Output}{Output}
\Input{Random Initialized/Pre-trained model $\bmmc{W}$, top-$K$, Target model size $c$, Cumulative threshold $\theta$}
\Output{Sparse Model $\hat{\bmmc{W}}$}
\BlankLine
\For{$t = 1, 2, \ldots, T$}{
    Update: $\bmmc{W}_t \gets \bmmc{W}_{t-1}+\delta \bmmc{W}$\;
    \If{$\mathrm{Param}(\bmmc{W}_t) > c$}{
        Compute $\hat{I}(\bm{W}_{\rm couple})$ for all $\bm{W}_{\rm couple} \in \bmmc{W}$ \tcp*[r]{$\triangleright$ Eq.~\ref{eq:importance}}
        Select top-$K$ $\bm{W}_{\rm couple}$ with smallest $\hat{I}$ to form $\{\bm{W}_{\rm couple}\}^{{\rm top}\text{-}K}$\;
        \ForEach{$\bm{W}_{\rm couple}$ in $\{\bm{W}_{\rm couple}\}^{{\rm top}\text{-}K}$}{
            Compute $\bm{\upsilon}$ of $\bm{W}_{\rm couple}$ \tcp*[r]{$\triangleright$ Eq.~\ref{eq:l2norm}}
            \While{$\sum\upsilon_i > \theta$}{
                Remove smallest $\upsilon_i$ and corresponding $i$-th row/column of $\bm{W}_{\rm couple}$\;
                Remove $i$-th row/column optimizer moments of $\bm{W}_{\rm couple}$\;
            }
        }
    }
}
\Return{$\hat{\bmmc{W}}$}
\end{algorithm}
\section{Experiments}

\subsection{Experiments on Pre-training (Train from Scratch)}
\textbf{Large Language Model (LLaMA):} For large language models, we follow the training settings in \cite{zhao2024galore} to train sparse LLaMA from scratch on the C4 dataset. The training employs the Galore optimizer with an initial learning rate of $0.01$ and a batch size of $512$.

\textbf{NLP Transformer (GPT-2):} We pre-train sparse NLP transformers by following the training settings in \cite{zhao2023inrank}. Specifically, we train sparse GPT-2 \cite{radford2019language} models from scratch on the WikiText-103 dataset \cite{merity2016pointer}, using the AdamW optimizer with an initial learning rate of $0.001$, training for 100 epochs and a batch size of $512$.

\textbf{Vision Transformer (DeiT):} We apply our approach to train sparse DeiT models \cite{pmlr-v139-touvron21a} from scratch on the ImageNet-1K dataset \cite{deng2009imagenet}. We use the same hyper-parameters as in \cite{chen2021chasing}, which include the AdamW optimizer with an initial learning rate of $0.0005$ and a batch size of $512$.

\begin{table}[H]
\centering
\caption{Evaluation of training LLaMA-1B models from scratch on C4 dataset on 13.1B tokens using 8$\times$A100 GPUs. Validation perplexity is provided, along with memory estimates for total parameters and optimizer states in BF16 format. The results for LoRA and ReLoRA are sourced from~\cite{zhao2024galore}.}
\vspace{-2mm}
\resizebox{1\linewidth}{!}{\begin{tabular}{l|c|c|c|c|c}
\toprule
Method & PPL & Training Time  & \# Params. (M) & \# Mem. (GB) & Throughput (tokens/s)\\
\midrule
Baseline         & 15.56       &51.1h      & 1339.08      & 7.80    &21786.49  \\
LoRA        & 19.21          & 125.4h           & 1339.08      & 6.17        &21786.49      \\
ReLoRA  & 18.33          & 125.6h           &1339.08      & 6.17       &21786.49      \\
Galore         &  15.64        & 60.6h   & 1339.08   & 4.38     &21786.49     \\
\rowcolor{gray!15} \textbf{EcoSpa} & 15.60 & \textbf{40.3h}   &\textbf{933.94}  &\textbf{3.88}   &\textbf{35211.27}       \\
\bottomrule
\end{tabular}}
\label{tbl:llama-1b}
\vspace{-6mm}
\end{table}

\begin{table}[H]
\centering
\caption{Evaluation of training LLaMA-7B models from scratch on the C4 dataset (1.4B tokens) using 8$\times$A100 GPUs.}
\vspace{-2mm}
\resizebox{\linewidth}{!}{\begin{tabular}{l|c|c|c|c|c}
\toprule
Method & PPL & Training Time  & \# Params. (M) & \# Mem. (GB) & Throughput (tokens/s)\\

\midrule
8-bit Galore         &  26.9       & 20.8h  & 6738.42  &17.86  & 9689.9           \\
8-bit SLtrain         &  27.6        & 31.4h  & 3144.73  &11.81  & 2204.6       \\
\rowcolor{gray!15} \textbf{8-bit EcoSpa} & 23.0 & \textbf{18.3h} &5046.96& 12.56&\textbf{10330.6}   \\
\bottomrule
\end{tabular}}
\label{tbl:llama-7b}
\vspace{-6mm}
\end{table}

\begin{table}[H]
\centering
\caption{Evaluation of training sparse/low-rank GPT-2 models from scratch on WikiText-103 using 8$\times$A100 GPUs. Validation perplexity is provided, along with memory estimates for total parameters and optimizer states in FP16 format.}
\vspace{-2mm}
\tiny
\resizebox{1\linewidth}{!}{\begin{tabular}{l|c|c|c|c|c}
\toprule
Method & PPL & Training Time  & \# Params. (M) & \# Mem. (MB) & Throughput (tokens/s)\\
\midrule
GPT2-Small        & 18.5          & 9.5h          & 124.4       & 711.8    & 47411.2     \\
In-Rank         & 18.9          & -     & 91.2      & 521.9             & -  \\
Monarch         & 20.7          & -          & 72     & 412.0       &-      \\
\rowcolor{gray!15} \textbf{EcoSpa} & \textbf{17.8} & \textbf{7.5h}    & \textbf{71.2} & \textbf{504.7}  & \textbf{80691.2}      \\
\midrule
GPT2-Medium       & 19.5          & 28.7h        & 354.8     & 2030.2      &36556.8      \\
In-Rank          & 20.2          & 18.0h         & 223.0   &1276.0        &-        \\
Monarch         & 20.3          & -         & 165    & 994.1       &-      \\
\rowcolor{gray!15}\textbf{EcoSpa} & \textbf{17.1} & \textbf{17.3h} & \textbf{158.5}   & \textbf{917.2}  &\textbf{52326.4}   \\
\bottomrule
\end{tabular}}
\label{tbl:gpt}
\vspace{-5mm}
\end{table}

\textbf{Comparison Results.} Table~\ref{tbl:llama-1b} presents the pre-training results for LLaMA-1B. Our approach reduces GPU memory usage by 50\%, decreases training time by $21\%$, and achieves a $1.6\times$ speedup in inference throughput without any performance loss. Compared to Galore, our approach reduces memory usage by 10\%, increases training speed by $1.5\times$, and boosts inference speed by $1.6\times$. The results for LoRA~\cite{hulora} and ReLoRA~\cite{lialin2023relora} are sourced from~\cite{zhao2024galore}. For low-rank methods, LoRA \cite{hu2021lora} fine-tunes pre-trained models using low-rank adaptors: \( W = W_0 + BA \), where \( W_0 \) is the fixed initial weights and \( BA \) is a learnable low-rank adaptor. For pre-training, \( W_0 \) is the full-rank initialization matrix. ReLoRA \cite{lialin2023relora} is a variant of LoRA designed for pre-training. It periodically merges \( BA \) into \( W \) and reinitializes \( BA \) with a reset on optimizer states and learning rate. EcoSpa surpasses these low-rank methods, reducing PPL by $3.6$ and $2.7$, respectively, while decreasing memory usage by $37\%$ and offering $1.6\times$ acceleration in inference throughput.

Table~\ref{tbl:llama-7b} presents the results of pre-training LLaMA-7B. Compared to SLTrain~\cite{han2024sltrain}, EcoSpa reduces training time by 41.7\% and achieves a 4.7$\times$ speedup in inference with better performance. Unlike SLTrain adopting the combination of low-rank factorization and unstructured sparsity  that cannot translate to actual speedup, the training process of EcoSpa is on the structured sparse models that enjoy measured training and inference speedup.

Table~\ref{tbl:gpt} compares our approach with the existing sparse and low-rank pre-training works. It is seen that EcoSpa achieves better performance than baseline and prior efforts using less training time. Specifically, our approach can train sparse GPT2-Small and GPT2-Medium models from scratch, with $1.7\times$ and $2.2\times$ model size reduction, respectively; and meanwhile, it brings 0.7 and 2.4 lower perplexity over the baseline models. 

Table~\ref{tbl:deit} lists the performance results of various sparse vision transformer pre-training methods. EcoSpa achieves a $0.8\%$ increase in top-1 accuracy for DeiT-Tiny and a $0.35\%$ increase for DeiT-Small over state-of-the-art solutions, along with greater model size reduction. 

\begin{table}[H]
\vspace{-5mm}
\centering
\caption{Results of training sparse DeiT models from scratch on ImageNet-1k. Results for SSP-Tiny and SSP-Small are sourced from \cite{chen2021chasing}. Inference speedup is measured on Nvidia RTX 3090 GPU.}
\vspace{-2mm}
\resizebox{\linewidth}{!}{\begin{tabular}{l|c|c|c|c|c}
\toprule
Method   & \# Params. (M) & \# Mem. (MB) &FLOPs Saving (\%) & Top-1 Acc. (\%)& Inference Speedup \\
\midrule
DeiT-Tiny  & 5.72 & 32.73 & -       & 71.80             & -       \\
SSP-Tiny & 4.21  & 24.09 & 23.69   & 68.59             & 1.12$\times $        \\
S$^2$ViTE-Tiny      & 4.21  & 24.09 & 23.69   & 70.12             & 1.12$\times $        \\

\rowcolor{gray!15}\textbf{EcoSpa}     & \textbf{3.95} & \textbf{22.60} &  \textbf{32.01}       & 70.92             & \textbf{1.20}$\times$         \\
\midrule
DeiT-Small      & 22.10 & 126.46       &  -       &79.78       & -       \\
SSP-Small & 14.60      &83.54         & 33.13   & 77.74             & 1.29$\times$         \\
S$^2$ViTE-Small      & 14.60      & 83.54        & 33.13   & 79.22             & 1.29$\times$         \\

\rowcolor{gray!15}\textbf{EcoSpa}     & \textbf{13.98}     & \textbf{79.99}    &\textbf{37.30}         & 79.57       & \textbf{1.29}$\times$         \\
\bottomrule
\end{tabular}}
\label{tbl:deit}
\vspace{-5mm}
\end{table}

\begin{table}[H]
\centering
\caption{Perplexity of compressed LLaMA2-7B on WikiText-2 with different target model sizes. SVD-LLM~\cite{wang2024svd} and SliceGPT~\cite{ashkboos2024slicegpt} are low-rank based methods. LLM Surgeon~\cite{van2023llm} is a pruning method, K-OBD~\cite{van2023llm}, as a baseline comparison method, uses Kronecker-factored curvature and only prunes without updating the remaining weights.}
\resizebox{\linewidth}{!}{\begin{tabular}{lc|c|c|c|c|>{\columncolor{gray!15}}c}
\toprule
Method & & K-OBD &SVD-LLM  & LLM Surgeon   &SliceGPT & \textbf{EcoSpa}   \\
\midrule
Training Time & & 16h58m, H100 & 15m, A100& 17h08m, H100 & 1h07m, H100  & \textbf{1h41m, A100}\\ 
\midrule
\multirow{4}{*}{\begin{tabular}[c]{@{}l@{}}PPL @ \\ Target Size\end{tabular}} & 80\% & 9.14  &7.94& 6.18 &6.64 & 6.36      \\
 & 70\% & 15.43 &9.56& 7.83  &8.12 & \textbf{7.66}      \\
 & 60\% & 28.03 &13.11& 10.39 &- & \textbf{10.24}     \\
 & 50\% & 46.64 &23.97& 15.38 &- & \textbf{14.02}   \\
\bottomrule
\end{tabular}}
\label{tbl:llama}
\vspace{-5mm}
\end{table}

\begin{table}[H]
\centering
\caption{Comparison of downstream zero-shot task performance of LLaMA2-7B model when trained on WikiText2 dataset.}
\resizebox{\linewidth}{!}{\begin{tabular}{lccccccc}
\toprule
                          & Target Size & PIQA  & WinoGrande & HellaSwag & ARC-e & ARC-c & Avg.  \\
\midrule
LLaMA2-7B                 & 100\%       & 79.11 & 69.06      & 75.99     & 74.58 & 46.25 & 69.00 \\
\midrule
\multirow{3}{*}{\textbf{EcoSpa}}     & 80\%        & 73.78 & 61.48      & 67.79     & 61.62 & 39.42 & 60.82 \\
                          & 75\%        & 71.93 & 60.69      & 64.69     & 54.38 & 35.15 & 57.37 \\
                          & 70\%        & 70.18 & 59.98      & 60.00     & 49.16 & 34.22 & 54.71 \\
\midrule
\multirow{3}{*}{SliceGPT} & 80\%        & 69.42 & 65.11      & 59.04     & 59.76 & 37.54 & 58.18 \\
                          & 75\%        & 66.87 & 63.38      & 54.16     & 58.46 & 34.56 & 55.48 \\
                          & 70\%        & 63.55 & 61.33      & 49.62     & 51.77 & 31.23 & 51.50 \\
\bottomrule
\end{tabular}}
\label{tbl:llama_zero}
\vspace{-3mm}
\end{table}
\subsection{Experiments on Fine-tuning LLaMA2-7B (Structured Pruning)}
\textbf{Experimental Setting.} We evaluate the pruning performance of EcoSpa on the pre-trained LLaMA2~\cite{touvron2023llama} models. We use WikiText-2~\cite{merity2016pointer} as the calibration dataset and evaluate the perplexity of the pruned model. We follow the same training settings adopted in ~\cite{van2023llm}, use 128 sequences with a sequence length of 2048 tokens from the training dataset, and evaluate perplexity on the standard test split. Additionally, we also evaluate the performance of the pruned models on downstream zero-shot tasks.

\textbf{Comparison Results.} Table \ref{tbl:llama} presents a comparison of the perplexity performance between EcoSpa and existing LLM pruning and low-rank factorization methods applied to LLaMA2-7B. Our approach consistently achieves lower perplexity across various target model size configurations compared to previous works. Additionally, Table \ref{tbl:llama_zero} illustrates the zero-shot task performance of the pruned LLaMA2-7B model. In comparison to SliceGPT~\cite{ashkboos2024slicegpt}, our method demonstrates improved results across different target model sizes, indicating its effectiveness.


\section{Conclusion}
In this paper, we propose EcoSpa, an efficient structured sparse training approach for transformers. By estimating the coupled weight matrix-wise importance and removing the coupled row/column pair during training, EcoSpa brings a significant reduction in training costs and model complexity with preserving high task performance. Experiments across various transformer models demonstrate the superior performance of EcoSpa in both pre-training and fine-tuning scenarios.

\newpage
{
\small
\bibliographystyle{plain}
\bibliography{reference}
}

\newpage
\appendix
\section*{Appendix}
\section{Theoretical Analysis}
\label{theo_sec}
\subsection{Problem Setup}

Let $\bm{A} \in \R^{m \times n}$ be a matrix whose columns are $\bm{a}_1, \dots, \bm{a}_n \in \R^m$.
Let $\bm{B} \in \R^{n \times p}$ be a matrix whose rows are $\bm{b}_1^T, \dots, \bm{b}_n^T$, where $\bm{b}_i \in \R^p$.
The product $\bm{A}\bm{B}$ can be expressed as a sum of outer products:
\begin{equation}
    \bm{P} \coloneqq \bm{A}\bm{B} = \sum_{i=1}^n \bm{a}_i \bm{b}_i^T
    \label{eq:product_sum}
\end{equation}
We analyze the error introduced by removing a single column from $\bm{A}$ and a single row from $\bm{B}$. Let $\bm{A}_{\text{new}}$ and $\bm{B}_{\text{new}}$ denote the matrices after removal. The objective is to compare the Frobenius norm of the error matrix, $\Delta \coloneqq \bm{P} - \bm{P}_{\text{new}} = \bm{A}\bm{B} - \bm{A}_{\text{new}}\bm{B}_{\text{new}}$, under two different sparsification strategies.

\subsection{Sparsification Strategies and Error Analysis}

We compare two methods for selecting which column/row pair to remove.

\textbf{1. Coupled Sparsification (Proposed)}

\begin{definition}[Coupled Sparsification]
Select the index $j^*$ that minimizes the Frobenius norm contribution of the corresponding outer product term:
\begin{equation}
    j^* = \argmin_{j \in \{1, \dots, n\}} \normF{\bm{a}_j \bm{b}_j^T} = \argmin_{j \in \{1, \dots, n\}} \normtwo{\bm{a}_j} \normtwo{\bm{b}_j}
    \label{eq:coupled_selection}
\end{equation}
Remove the column $\bm{a}_{j^*}$ from $\bm{A}$ to obtain $\bm{A}_{\text{new}}^{(C)}$ and the row $\bm{b}_{j^*}^T$ from $\bm{B}$ to obtain $\bm{B}_{\text{new}}^{(C)}$. The resulting matrices have dimensions $m \times (n-1)$ and $(n-1) \times p$, respectively. The new product is implicitly defined by keeping the original pairings for the remaining terms:
\begin{equation}
    \bm{P}_{\text{new}}^{(C)} \coloneqq \sum_{i \neq j^*} \bm{a}_i \bm{b}_i^T
\end{equation}
\end{definition}

\begin{proposition}[Error under Coupled Sparsification]
The error matrix introduced by Coupled Sparsification is exactly the removed outer product term:
\begin{equation}
    \Delta_1 \coloneqq \bm{P} - \bm{P}_{\text{new}}^{(C)} = \sum_{i=1}^n \bm{a}_i \bm{b}_i^T - \sum_{i \neq j^*} \bm{a}_i \bm{b}_i^T = \bm{a}_{j^*} \bm{b}_{j^*}^T
    \label{eq:delta1_matrix}
\end{equation}
The Frobenius norm of this error is:
\begin{equation}
    \normF{\Delta_1} = \normF{\bm{a}_{j^*} \bm{b}_{j^*}^T} = \normtwo{\bm{a}_{j^*}} \normtwo{\bm{b}_{j^*}} = \min_{j \in \{1, \dots, n\}} \normtwo{\bm{a}_j} \normtwo{\bm{b}_j}
    \label{eq:delta1_norm}
\end{equation}
\end{proposition}
\begin{proof}
Equation \eqref{eq:delta1_matrix} follows directly from the definition of $\bm{P}$ in \eqref{eq:product_sum} and $\bm{P}_{\text{new}}^{(C)}$. Equation \eqref{eq:delta1_norm} uses the property $\normF{\bm{u}\bm{v}^T} = \normtwo{\bm{u}}\normtwo{\bm{v}}$ and the selection criterion from \eqref{eq:coupled_selection}.
\end{proof}

\begin{remark}
Coupled Sparsification minimizes the Frobenius norm of the error matrix $\Delta_1$ for a single pair removal, based on the greedy selection of the pair $(\bm{a}_j, \bm{b}_j)$ that contributes the least to the Frobenius norm.
\end{remark}

\textbf{2. Individual Sparsification (Baseline)}

\begin{definition}[Individual Sparsification]
Select the index $k^*$ corresponding to the column of $\bm{A}$ with the smallest L2 norm, and the index $l^*$ corresponding to the row of $\bm{B}$ (or column of $\bm{B}^T$) with the smallest L2 norm:
\begin{align}
    k^* &= \argmin_{k \in \{1, \dots, n\}} \normtwo{\bm{a}_k} \label{eq:individual_selection_k} \\
    l^* &= \argmin_{l \in \{1, \dots, n\}} \normtwo{\bm{b}_l} \label{eq:individual_selection_l}
\end{align}
Remove column $\bm{a}_{k^*}$ from $\bm{A}$ to obtain $\bm{A}_{\text{new}}^{(I)}$ and row $\bm{b}_{l^*}^T$ from $\bm{B}$ to obtain $\bm{B}_{\text{new}}^{(I)}$. The resulting matrices have dimensions $m \times (n-1)$ and $(n-1) \times p$, respectively. The new product $\bm{P}_{\text{new}}^{(I)}$ is formed by multiplying these modified matrices:
\begin{equation}
    \bm{P}_{\text{new}}^{(I)} \coloneqq \bm{A}_{\text{new}}^{(I)} \bm{B}_{\text{new}}^{(I)}
\end{equation}
Let the columns of $\bm{A}_{\text{new}}^{(I)}$ be $\bm{a}'_p$ and the rows of $\bm{B}_{\text{new}}^{(I)}$ be $(\bm{b}'_p)^T$ for $p=1, \dots, n-1$. These are the original columns/rows excluding $\bm{a}_{k^*}$ and $\bm{b}_{l^*}$, implicitly re-indexed. The product is:
\begin{equation}
    \bm{P}_{\text{new}}^{(I)} = \sum_{p=1}^{n-1} \bm{a}'_p (\bm{b}'_p)^T
\end{equation}
\end{definition}

\begin{proposition}[Error under Individual Sparsification]
The error matrix introduced by Individual Sparsification is:
\begin{equation}
    \Delta_2 \coloneqq \bm{P} - \bm{P}_{\text{new}}^{(I)} = \sum_{i=1}^n \bm{a}_i \bm{b}_i^T - \sum_{p=1}^{n-1} \bm{a}'_p (\bm{b}'_p)^T
    \label{eq:delta2_matrix}
\end{equation}
This error term $\Delta_2$ can be complex, especially when $k^* \neq l^*$.
\end{proposition}

\begin{remark}[Mismatch Effect]
If $k^* = l^*$, then Individual Sparsification happens to select the same index. In this case, $\bm{A}_{\text{new}}^{(I)} = \bm{A}_{\text{new}}^{(C)}$ and $\bm{B}_{\text{new}}^{(I)} = \bm{B}_{\text{new}}^{(C)}$, leading to $\Delta_2 = \bm{a}_{k^*} \bm{b}_{k^*}^T$. The error norm is $\normF{\Delta_2} = \normtwo{\bm{a}_{k^*}} \normtwo{\bm{b}_{k^*}}$.

However, if $k^* \neq l^*$, the matrix product $\bm{P}_{\text{new}}^{(I)} = \bm{A}_{\text{new}}^{(I)} \bm{B}_{\text{new}}^{(I)}$ involves multiplying columns and rows that were not originally paired. For example, if $k^*=1, l^*=2, n=3$, then $\bm{A}_{\text{new}}^{(I)} = [\bm{a}_2, \bm{a}_3]$ and $\bm{B}_{\text{new}}^{(I)} = \begin{bsmallmatrix} \bm{b}_1^T \\ \bm{b}_3^T \end{bsmallmatrix}$. The product is $\bm{P}_{\text{new}}^{(I)} = \bm{a}_2 \bm{b}_1^T + \bm{a}_3 \bm{b}_3^T$.
The original product was $\bm{P} = \bm{a}_1 \bm{b}_1^T + \bm{a}_2 \bm{b}_2^T + \bm{a}_3 \bm{b}_3^T$.
The error is $\Delta_2 = \bm{P} - \bm{P}_{\text{new}}^{(I)} = \bm{a}_1 \bm{b}_1^T + \bm{a}_2 \bm{b}_2^T - \bm{a}_2 \bm{b}_1^T$.

In general, $\Delta_2$ can be expressed as:
\begin{equation}
    \Delta_2 = \underbrace{\bm{a}_{k^*} \bm{b}_{k^*}^T}_{\text{Term related to removed } \bm{a}_{k^*}} + \underbrace{\bm{a}_{l^*} \bm{b}_{l^*}^T}_{\text{Term related to removed } \bm{b}_{l^*}} + \underbrace{\left( \sum_{i \neq k^*, l^*} \bm{a}_i \bm{b}_i^T - \bm{P}_{\text{new}}^{(I)} \right)}_{\text{Mismatch Effect (if } k^* \neq l^* \text{)}}
\end{equation}
The "Mismatch effect" arises because the structure $\sum \bm{a}_i \bm{b}_i^T$ is broken by removing components based on potentially different indices $k^*$ and $l^*$. Calculating $\bm{P}_{\text{new}}^{(I)}$ involves multiplying the remaining $n-1$ columns of $\bm{A}_{\text{new}}^{(I)}$ with the remaining $n-1$ rows of $\bm{B}_{\text{new}}^{(I)}$, creating cross-terms that differ from the original summation structure. This structure makes $\normF{\Delta_2}$ difficult to analyze directly and prevents it from directly optimizing a simple objective related to $\normtwo{\bm{a}_{k^*}}$ or $\normtwo{\bm{b}_{l^*}}$.
\label{rem:mismatch}
\end{remark}

\subsection{Theoretical Comparison}

\begin{proposition}[Probabilistic Error Comparison]
The error introduced by Coupled Sparsification, $\normF{\Delta_1}$, is probabilistically likely to be smaller than the error introduced by Individual Sparsification, $\normF{\Delta_2}$.
\end{proposition}

\begin{proof}[Argument Sketch]
1.  Coupled Sparsification, by definition \eqref{eq:coupled_selection}, selects the pair $(j^*)$ such that the Frobenius norm of the removed term, $\normF{\bm{a}_{j^*} \bm{b}_{j^*}^T}$, is minimized. This minimized value is precisely the Frobenius norm of the error, $\normF{\Delta_1}$ (Equation \eqref{eq:delta1_norm}). The method directly optimizes an upper bound related to the reconstruction error for the specific structure $\bm{A}\bm{B} = \sum \bm{a}_i \bm{b}_i^T$.

2.  Individual Sparsification selects indices $k^*$ and $l^*$ based on minimizing individual vector norms $\normtwo{\bm{a}_k}$ and $\normtwo{\bm{b}_l}$ independently.
    - If $k^*=l^*$, the error norm is $\normF{\Delta_2} = \normtwo{\bm{a}_{k^*}} \normtwo{\bm{b}_{k^*}}$. Since $j^*$ minimizes the product $\normtwo{\bm{a}_j} \normtwo{\bm{b}_j}$, we have $\normF{\Delta_1} \le \normF{\Delta_2}$ in this specific case.
    - If $k^* \neq l^*$, the error $\Delta_2$ includes the complex "Mismatch effect" described in Remark \ref{rem:mismatch}. The selection criteria \eqref{eq:individual_selection_k} and \eqref{eq:individual_selection_l} do not directly minimize $\normF{\Delta_2}$. The mismatch term introduces components unrelated to the individual norms being minimized, breaking the direct link between the optimization objective and the resulting error norm.

3.  Because Coupled Sparsification directly minimizes the quantity $\normtwo{\bm{a}_j} \normtwo{\bm{b}_j}$ which equals the error norm $\normF{\Delta_1}$, while Individual Sparsification minimizes separate quantities ($\normtwo{\bm{a}_k}$, $\normtwo{\bm{b}_l}$) leading to a complex error $\Delta_2$ (often involving mismatch effects) that isn't directly minimized, it is probabilistically expected that $\normF{\Delta_1} \le \normF{\Delta_2}$. The individual strategy optimizes components in isolation, failing to account for their coupled contribution to the product $\bm{A}\bm{B}$ and the structure of the resulting error upon removal, especially when $k^* \neq l^*$.
\end{proof}

\subsection{Empirical Validation}

To verify this theoretical advantage, empirical tests were conducted.
\begin{enumerate}
    \item Random matrices $\bm{A}$ and $\bm{B}$ were generated for various dimensions ($n \times n$ with $n \in \{100, 500, 1000, 5000\}$).
    \item For each dimension, 5000 random matrix pairs were generated.
    \item The Frobenius norm errors, $\normF{\Delta_1}$ (Coupled) and $\normF{\Delta_2}$ (Individual), were computed after removing one column/row pair using each strategy.
    \item The probability $\mathbb{P}(\normF{\Delta_1} < \normF{\Delta_2})$ was estimated from the trials.
\end{enumerate}
The results are summarized in Table \ref{tab:empirical_results}.

\begin{table}[h!]
    \centering
    \caption{Empirical Comparison of Frobenius Error Norms}
    \label{tab:empirical_results}
    \begin{tabular}{cccc}
        \toprule
        Matrix Size ($n$) & Avg. $\normF{\Delta_1}$ (Coupled) & Avg. $\normF{\Delta_2}$ (Individual) & $\mathbb{P}(\normF{\Delta_1} < \normF{\Delta_2})$ (\%) \\
        \midrule
        $100 \times 100$   & 76.4    & 757.9   & 99.0 \\
        $500 \times 500$   & 434.4   & 8498.9  & 99.7 \\
        $1000 \times 1000$ & 900.5   & 23931.2 & 99.9 \\
        $5000 \times 5000$ & 4743.7  & 266465.9 & 100.0 \\
        \bottomrule
    \end{tabular}
\end{table}

The empirical results strongly support the theoretical analysis, showing that Coupled Sparsification yields significantly lower error with very high probability ($>99\%$) compared to Individual Sparsification across different matrix dimensions.


\section{Theoretical Motivation for Coupled Sparsification in FFNs}
\label{theo_sec_ffn}
\subsection{Derivation from Pruning Error Approximation (for GELU)}
Our method, \textbf{Coupled Sparsification}, removes the $j$-th column of $\mathbf{W}_{in}$ and the $j$-th row of $\mathbf{W}_{out}$ together. The true error this introduces for an input $\mathbf{X}$ is:
\begin{equation}
    \text{Error}(j, \mathbf{X}) = \| \text{GELU}(\mathbf{X} \mathbf{W}_{in}) \mathbf{w}_{out, j} \|_F
\end{equation}
To analyze this without input-dependency, we approximate the error using the first-order Taylor expansion of GELU around zero ($\text{GELU}(u) \approx c \cdot u$). Applying this approximation yields a bound on the true error:
\begin{equation}
    \begin{aligned}
    \text{Error}(j, \mathbf{X}) &\approx \| c \cdot (\mathbf{X} \mathbf{w}_{in, j}) \mathbf{w}_{out, j} \|_F \\
    &= c \cdot \| \mathbf{X} (\mathbf{w}_{in, j} \mathbf{w}_{out, j}) \|_F \\
    &\le c \cdot \| \mathbf{X} \|_F \cdot \| \mathbf{w}_{in, j} \mathbf{w}_{out, j} \|_F \\
    &= (c \cdot \| \mathbf{X} \|_F) \cdot (\| \mathbf{w}_{in, j} \|_2 \cdot \| \mathbf{w}_{out, j} \|_2)
    \end{aligned}
\end{equation}

This derivation shows that the true error is bounded by a term proportional to our \textbf{Coupled Metric}. The data-dependent part is a common scaling factor, leaving our metric as the decisive factor for ranking. This justifies why the \textbf{Coupling Effect} exists: the coupled weight matrix drives the dominant, first-order term of the FFN computation.

\subsection{Applicability to Other FFN Types and the Coupled Effect}
While the mathematical forms of other FFN layer types, such as ReLU or SwiGLU, differ from GELU and may not present a direct linear term in the same manner, our approach is designed to adapt to these structures. A key insight is that for many non-linear activations (e.g., ReLU, SiLU, GELU), significant input magnitudes can lead to negligible outputs (e.g., for negative inputs).

Our method addresses this by focusing on the \textit{coupled influence} of the input and output weight matrices. For SwiGLU, as elaborated in Footnote 2, we establish a coupled structure by defining $\mathbf{W}_1 = \mathbf{W}_{\text{gate}} \odot \mathbf{W}_{\text{up}}$ and $\mathbf{W}_2 = \mathbf{W}_{\text{down}}$. This formulation is chosen to capture the critical functional relationships and information flow between these interdependent weight matrices. By assessing importance through this \textbf{coupled effect}, we inherently account for how the non-linearity modulates a neuron’s ultimate contribution, leading to a more accurate importance evaluation.

\subsection{Coupled Sparsification V.S. Uncoupled Sparsification in FFN}
To empirically validate the effectiveness of our \textbf{Coupled Sparsification} method for FFN matrices, we conducted direct comparisons against \textbf{Uncoupled Sparsification} on both LLaMA2 (with SwiGLU) and GPT-2 (with GELU) models, evaluated on the WikiText-2 dataset. As presented in Table~\ref{tab:coupled_vs_uncoupled}, Coupled Sparsification dramatically outperforms Uncoupled Sparsification.

\begin{table}[!htbp]
\caption{Coupled Sparsification V.S. Uncoupled Sparsification in FFN.}
\centering
\begin{tabular}{cccccc}
\toprule
\multicolumn{1}{l}{}  & \textbf{Sparsity} & 30\%  & 50\%  & 70\%   & 90\% \\ 
\midrule
\multirow{2}{*}{\begin{tabular}[c]{@{}c@{}}LLaMA2 (SwiGLU) \\ (PPL=9.4)\end{tabular}} & Coupled   & 142.2 & 730.7 & 2115.6 & 6481.6   \\
  & Uncoupled & 187062.7  & 208483.4  & 261173.1   & 173291.2 \\ 
\midrule
\multirow{2}{*}{\begin{tabular}[c]{@{}c@{}}GPT2-m (GELU) \\ (PPL=21.4)\end{tabular}}  & Coupled   & 19137.1   & 24296.2   & 16157.3& 28261.5  \\
  & Uncoupled & 4205783.8 & 1158156.8 & 12031555.3 & 7131925831.1 \\
\bottomrule
\end{tabular}
\label{tab:coupled_vs_uncoupled}
\end{table}

\section{Hyper-parameters}
\label{hyper_sec}
\subsection{Selection of $K$ and $\theta$.} 
As described in Algorithm \ref{alg:trainig}, $K$ determines the percentage of $\bm{W}_{couple}$'s that will be sparsified in each epoch, and threshold $\theta$ impacts the amount of to-be-removed coupled row/column pairs within those unimportant  $\bm{W}_{couple}$'s. As shown in Fig. \ref{fig:vit_hyper}, when applying EcoSpa at the pre-training stage, larger $\theta$ brings better training performance, while the model is less sensitive to the change of $K$. On the other hand, Fig. \ref{fig:nlp_hyper} shows that it is better to use smaller $K$ at the fine-tuning stage, while the selection of $\theta$ is less significant in this scenario. We hypothesize that such difference might be due to the existence of a pre-trained model in the fine-tuning process since a pre-trained model typically has more diverse distribution of $\bm{W}_{couple}$ than the model being sparsely trained from random initialization, making the identification of unimportant $\bm{W}_{couple}$ more effective. Therefore, considering $K$ cannot be too small (otherwise, it is challenging to meet the model size budget (see the change of parameters (dashed curve) in Fig. \ref{fig:nlp_hyper}), we set $K=30\%$ and $\theta=90\%$ in our experiments.

\begin{figure}[!htbp]
    \centering
    \begin{subfigure}
        \centering
        \includegraphics[width=1\linewidth]{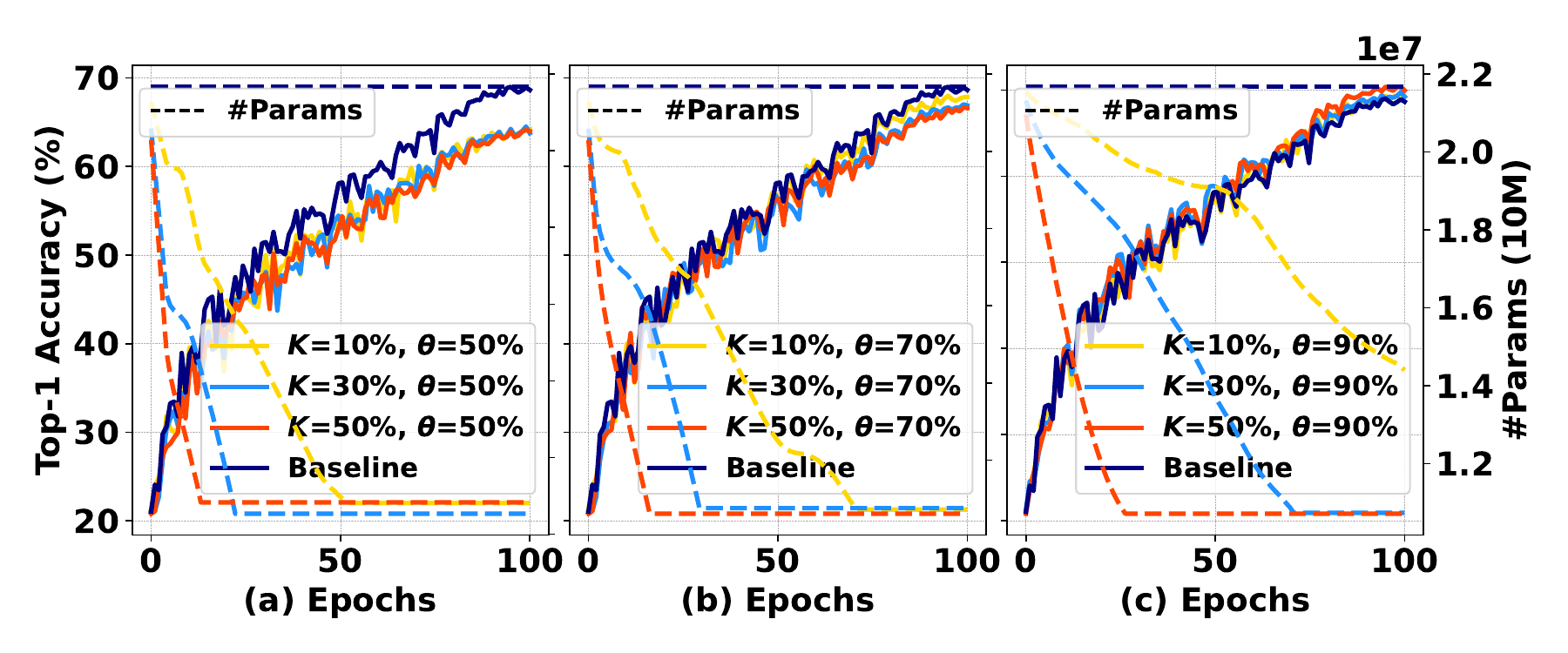}
        \caption{Pre-training sparse DeiT-Small on CIFAR-10 with various $K$ and $\theta$. The target model size is 11M.}
        \label{fig:vit_hyper} 
    \end{subfigure}
    \begin{subfigure}
        \centering
        \includegraphics[width=1\linewidth]{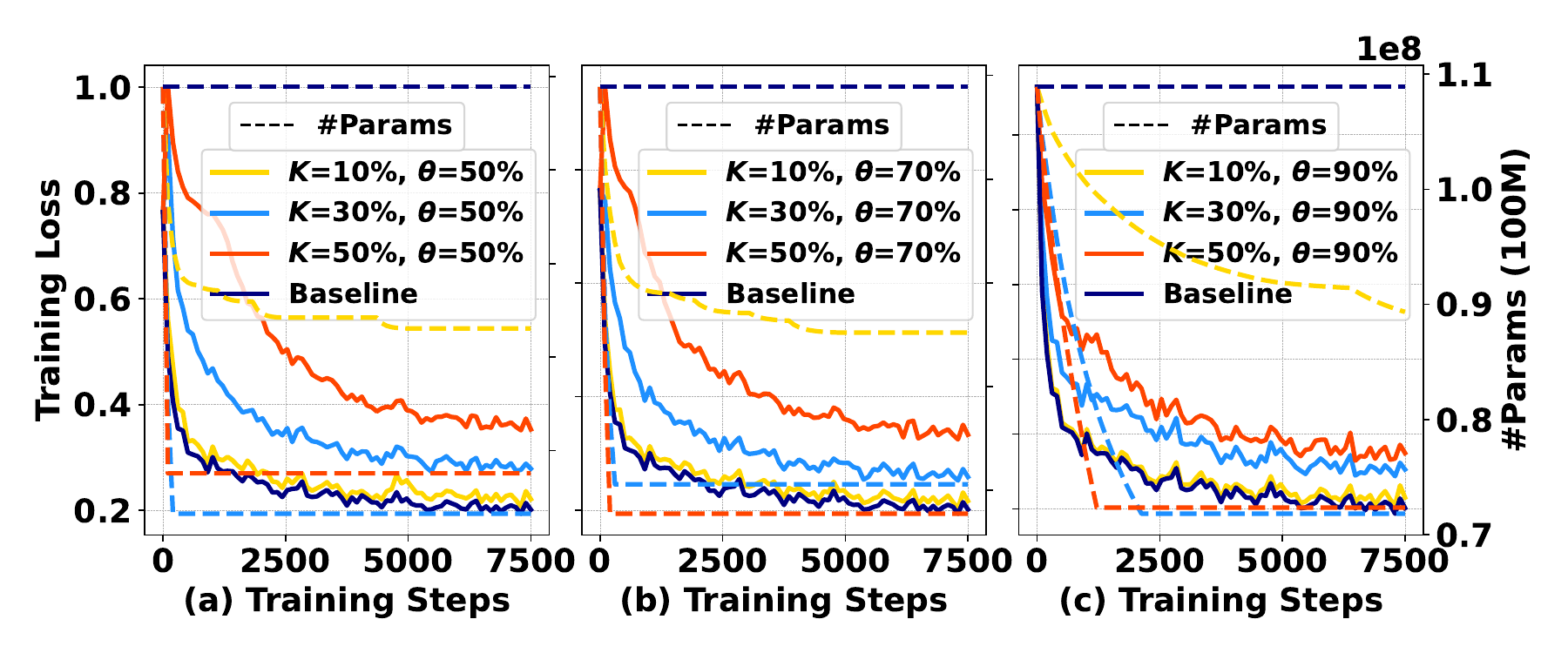}
        \caption{Fine-tuning Pre-trained sparse BERT-Base on SQuAD with various $K$ and $\theta$. The target model size is 72.6M.}
        \label{fig:nlp_hyper}
    \end{subfigure}
\end{figure}

\begin{figure}[!htbp]
    \centering
    \includegraphics[width=1\linewidth]{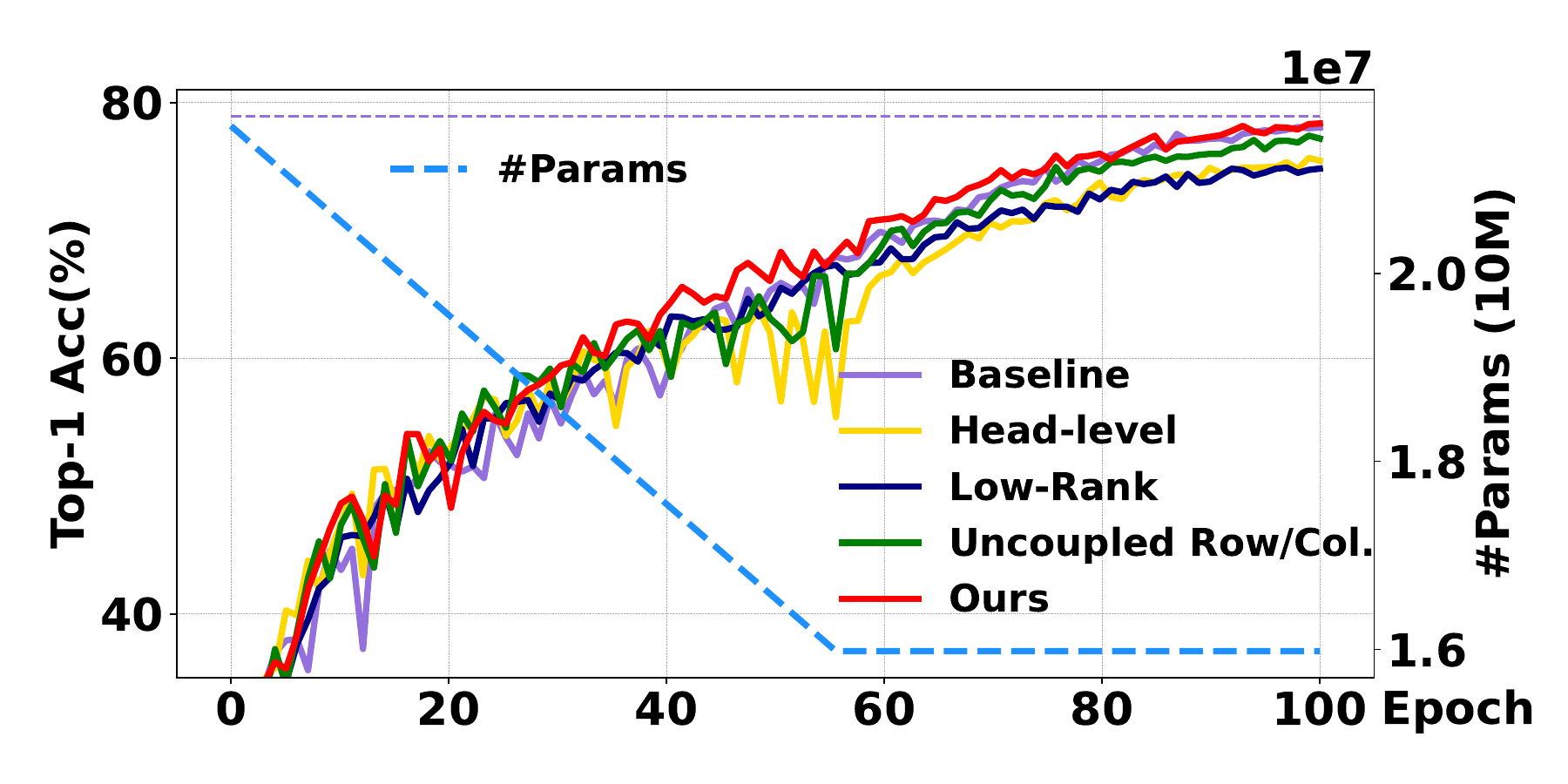}
    \caption{Pre-training DeiT-Small model on CIFAR-10 dataset using different $\bm{W}_{couple}$ compression methods.  All the methods remove the same number of parameters within the same $\bm{W}_{couple}$'s in each epoch.}
    \label{fig:diff_contraction}
    \vspace{-3mm}
\end{figure}

\subsection{Importance of Coupled Sparsification.} 
Fig. \ref{fig:diff_contraction} compares the training performance using different methods to compress $\bm{W}_{couple}$. It is seen that our proposed coupled row/column-wise sparsification scheme achieves the best performance. In particular, it outperforms the uncoupled row/column-based solution, demonstrating the importance of removing the row/column of $\bm{W}_{1}$ and $\bm{W}_{2}$ in a coupled way. Notice that though our approach is inspired by tSVD, it achieves better performance than directly applying tSVD on $\bm{W}_{couple}$. This is because SVD not only changes the structure of the model but also alters the numerical distribution of the original model weights. Consequently, the optimizers that keep track of moment information, \emph{e.g.}, Adam \cite{kingma2014adam}, cannot work well since they will not be able to use previously accumulated information, thereby affecting the model performance.

\begin{table}[!htbp] 
    \centering
    \caption{Overview of Experimental Setups and Hyper-parameters} %
    \label{tab:experimental_setup}
    \begin{tabular}{lllcc} 
        \toprule 
        \textbf{Experiments} & \textbf{Model} & \textbf{Dataset} & \textbf{top - $K$(\%)} & \textbf{$\theta$(\%)} \\ 
        \midrule 
        Table~\ref{tbl:llama-1b} & LLaMA-1B & C4 & 30 & 90 \\
        Table~\ref{tbl:llama-7b} & LLaMA-7B & C4 & 30 & 90 \\
        Table~\ref{tbl:gpt} & GPT & WikiText & 30 & 90 \\
        Table~\ref{tbl:deit} & DeiT & ImageNet & 30 & 90 \\
        Table~\ref{tbl:llama} & LLaMA2-7B@80\% & WikiText & 30 & 95 \\
        Table~\ref{tbl:llama} & LLaMA2-7B@70\% & WikiText & 30 & 90 \\
        Table~\ref{tbl:llama} & LLaMA2-7B@60\% & WikiText & 30 & 90 \\
        Table~\ref{tbl:llama} & LLaMA2-7B@50\% & WikiText & 30 & 85 \\
        Table~\ref{tbl:llama_zero} & LLaMA2-7B & WikiText & 30 & 95 \\
        \bottomrule 
    \end{tabular}
\end{table}

\subsection{Observation of Compressed Model.}
We analyze the resulting structure of the compressed models. Intriguing patterns emerge, revealing how the coupled sparsification strategy interacts differently with components based on model architecture and task. These observations, detailed below and summarized in Tables~\ref{tab:deit_dims} and \ref{tab:gpt2_dims}, provide insights into the method's adaptive nature.

\begin{itemize}
    \item \textbf{Language Models (e.g., GPT-2):} 
    In language models like GPT-2, the primary goal is sequence modeling and maintaining contextual coherence. We observe that EcoSpa tends to preserve the dimensions of the value projection ($\bm{W}^V$) and the final output projection ($\bm{W}^O$) matrices within the attention blocks. These components are crucial for integrating contextual information and propagating it through the network. Conversely, the query ($\bm{W}^Q$) and key ($\bm{W}^K$) matrices, which predominantly determine the attention patterns (implicitly via $\bm{W}^Q \bm{W}^{K^T}$), undergo more aggressive compression. This suggests that EcoSpa identifies greater redundancy within the attention pattern formation mechanism while prioritizing the preservation of the value pathway for maintaining coherence, aligning with findings on functional roles within transformer circuits~\cite{elhage2021mathematical}. 

    \item \textbf{Vision Models (e.g., DeiT):} 
    In vision transformers aimed at classification tasks like DeiT, a different pattern emerges. While EcoSpa compresses matrices throughout most of the network layers (Blocks 1-10 in DeiT-Tiny, see Table~\ref{tab:deit_dims}), the dimensions in the final, higher-level blocks (Blocks 11-12) often remain largely unchanged or are compressed less aggressively. These later layers are typically responsible for consolidating abstract features critical for the final classification decision. The observed behavior indicates that EcoSpa implicitly safeguards these high-level representations by reducing compression in deeper layers, while readily exploiting redundancies in the earlier feature extraction layers.
\end{itemize}

These distinct behaviors across model types underscore EcoSpa's ability to adaptively apply sparsity based on the implicit functional importance of different components, preserving critical pathways while effectively pruning less essential dimensions identified through the coupled analysis.

\begin{table}[!htbp] 
\centering
\caption{Dimensionality changes in DeiT-Tiny (3 Heads per block) after applying EcoSpa block-wise.}
\label{tab:deit_dims}
\resizebox{\textwidth}{!}{
\begin{tabular}{@{}lccc@{}} 
\toprule
Layer / Block & $\bm{W}^Q, \bm{W}^K$ Dim. & $\bm{W}^V, \bm{W}^O$ Dim. & $\bm{W}^{\text{in}}, \bm{W}^{\text{out}}$ (FFN) Dim. \\ 
\midrule
Original      & $192 \times 64$ & $192 \times 64$ & $192 \times 768$ \\
\midrule
\multicolumn{4}{l}{\textit{After EcoSpa}} \\ 
Block 1       & $192 \times 37$ & $192 \times 41$ & $192 \times 399$ \\
Block 2       & $192 \times 39$ & $192 \times 43$ & $192 \times 393$ \\
Block 3       & $192 \times 40$ & $192 \times 42$ & $192 \times 501$ \\
Block 4       & $192 \times 41$ & $192 \times 50$ & $192 \times 451$ \\
Block 5       & $192 \times 43$ & $192 \times 50$ & $192 \times 419$ \\
Block 6       & $192 \times 49$ & $192 \times 50$ & $192 \times 402$ \\
Block 7       & $192 \times 56$ & $192 \times 50$ & $192 \times 391$ \\
Block 8       & $192 \times 55$ & $192 \times 56$ & $192 \times 379$ \\
Block 9       & $192 \times 56$ & $192 \times 57$ & $192 \times 366$ \\
Block 10      & $192 \times 50$ & $192 \times 56$ & $192 \times 387$ \\
Block 11      & $192 \times 64$ & $192 \times 64$ & $192 \times 768$ \\
Block 12      & $192 \times 64$ & $192 \times 64$ & $192 \times 768$ \\
\bottomrule
\end{tabular}%
} 
\end{table}

\begin{table}[!htbp]
\centering
\caption{Dimensionality changes in GPT-2-Small (12 Heads per block) after applying EcoSpa block-wise.}
\label{tab:gpt2_dims}
\resizebox{\textwidth}{!}{
\begin{tabular}{@{}lccc@{}}
\toprule
Layer / Block & $\bm{W}^Q, \bm{W}^K$ Dim. & $\bm{W}^V, \bm{W}^O$ Dim. & $\bm{W}^{\text{in}}, \bm{W}^{\text{out}}$ (FFN) Dim. \\ 
\midrule
Original      & $768 \times 64$ & $768 \times 64$ & $768 \times 3072$ \\ 
\midrule
\multicolumn{4}{l}{\textit{After EcoSpa}} \\ 
Block 1       & $768 \times 64$ & $768 \times 64$ & $768 \times 2064$ \\
Block 2       & $768 \times 25$ & $768 \times 64$ & $768 \times 1337$ \\
Block 3       & $768 \times 21$ & $768 \times 64$ & $768 \times 1259$ \\
Block 4       & $768 \times 21$ & $768 \times 64$ & $768 \times 1643$ \\
Block 5       & $768 \times 37$ & $768 \times 64$ & $768 \times 1639$ \\
Block 6       & $768 \times 30$ & $768 \times 64$ & $768 \times 1613$ \\
Block 7       & $768 \times 20$ & $768 \times 64$ & $768 \times 1629$ \\
Block 8       & $768 \times 20$ & $768 \times 64$ & $768 \times 1104$ \\
Block 9       & $768 \times 20$ & $768 \times 64$ & $768 \times 1073$ \\
Block 10      & $768 \times 21$ & $768 \times 64$ & $768 \times 1589$ \\
Block 11      & $768 \times 26$ & $768 \times 64$ & $768 \times 1574$ \\
Block 12      & $768 \times 21$ & $768 \times 64$ & $768 \times 2044$ \\
\bottomrule
\end{tabular}%
} 
\end{table}
\section{Broader Impacts}
Our research on sparse training for transformer models enables more computationally efficient and environmentally sustainable training and deployment of AI models. Developing sparse training methods for Transformer models can significantly accelerate AI training and inference, fostering innovation while reducing energy consumption. This not only democratizes the development and utilization of AI but also aligns with global efforts towards sustainable computing practices, mitigating the environmental impact associated with training resource-intensive neural networks. As AI permeates various domains, optimizations like sparse training will play a crucial role in striking a balance between model performance and environmental responsibility, ensuring the responsible advancement of this technology.



\end{document}